\documentclass[letterpaper,letterpaper]{article}
\usepackage{helvet}
\usepackage{courier}
\usepackage[latin9]{inputenc}
\usepackage{float}
\usepackage{multirow}
\usepackage{amsmath}
\usepackage{amsthm}
\usepackage{graphicx}
\usepackage[authoryear]{natbib}

\makeatletter

\pdfpageheight\paperheight
\pdfpagewidth\paperwidth

\providecommand{\tabularnewline}{\\}

\theoremstyle{plain}
\newtheorem{thm}{\protect\theoremname}
\theoremstyle{definition}
\newtheorem{defn}[thm]{\protect\definitionname}
\theoremstyle{plain}
\newtheorem{prop}[thm]{\protect\propositionname}

\usepackage{aaai24}  
\usepackage{times}  
\usepackage{helvet}  
\usepackage{courier}  
\usepackage[hyphens]{url}  
\usepackage{graphicx} 
\def\UrlFont{\rm}  
\usepackage{natbib}  
\usepackage{caption} 
\usepackage{algorithm}
\usepackage{algorithmic}

\usepackage{newfloat}
\usepackage{listings}

\usepackage{amsmath}
\usepackage{amsfonts}

\DeclareCaptionStyle{ruled}{labelfont=normalfont,labelsep=colon,strut=off} 
\floatstyle{ruled}
\newfloat{listing}{tb}{lst}{}
\floatname{listing}{Listing}
\title{Root Cause Explanation of Outliers under Noisy Mechanisms}
\author {
    Phuoc Nguyen\textsuperscript{\rm 1},
    Truyen Tran\textsuperscript{\rm 1},
    Sunil Gupta\textsuperscript{\rm 1},
    Thin Nguyen\textsuperscript{\rm 1},
    Svetha Venkatesh\textsuperscript{\rm 1}
}
\affiliations {
    \textsuperscript{\rm 1}Deakin University\\
    firstname.lastname@deakin.edu.au
}

\providecommand{\definitionname}{Definition}
\providecommand{\propositionname}{Proposition}
\providecommand{\theoremname}{Theorem}

\makeatother

\providecommand{\definitionname}{Definition}
\providecommand{\propositionname}{Proposition}
\providecommand{\theoremname}{Theorem}

\begin{document}
\maketitle

\global\long\def\argmax{\operatornamewithlimits{argmax}}%
\global\long\def\argmin{\operatornamewithlimits{argmin}}%

\begin{abstract}
Identifying root causes of anomalies in causal processes is vital
across disciplines. Once identified, one can isolate the root causes
and implement necessary measures to restore the normal operation.
Causal processes are often modelled as graphs with entities being
nodes and their paths/interconnections as edge. Existing work only
consider the contribution of nodes in the generative process, thus
can not attribute the outlier score to the edges of the mechanism
if the anomaly occurs in the connections. In this paper, we consider
both individual edge and node of each mechanism when identifying the
root causes. We introduce a noisy functional causal model to account
for this purpose. Then, we employ Bayesian learning and inference
methods to infer the noises of the nodes and edges. We then represent
the functional form of a target outlier leaf as a function of the
node and edge noises. Finally, we propose an efficient gradient-based
attribution method to compute the anomaly attribution scores which
scales linearly with the number of nodes and edges. Experiments on
simulated datasets and two real-world scenario datasets show better
anomaly attribution performance of the proposed method compared to
the baselines. Our method scales to larger graphs with more nodes
and edges.
\end{abstract}

\section{Introduction}

Understanding the root causes behind anomalies in complex network
systems holds significant importance across various disciplines, ranging
from science to industry \citep{dhaou2021causal,pool2020lumos,yi2021semi}.
Once the causes are identified, one can isolate and implement necessary
measures to restore the normal operation of the process. Early fault
recognition would prevent costly damage to operations, services, and
products. For example, in complex network systems within modern manufacturing
industries and information services, the cost of system failure is
notably high \citep{ni2017ranking}, reaching as much as $\$20,000$
per minute of downtime in an automotive manufacturing plant \citep{djurdjanovic2003watchdog}.
The monitoring data and logs obtained from the processing nodes of
these systems often contain noise attributed to random fluctuations
in nodes and the links between them. In communication systems, for
instance, random delays between nodes can arise from bandwidth limitations
and network overhead \citep{zhang2001stability}. Given the intricate
dependencies between monitoring nodes and the substantial volume of
data involved, manual analysis of root causes becomes impractical.

Recent methods for root cause analysis (RCA) \citep{budhathoki2021did,budhathoki2022causal}
are based on a given causal structure of the system and a learned
functional causal model \citep{peters2017elements} to explain the
anomalous observations at a leaf node. These methods work by first
detecting the anomalous leaf, then utilising the causal structure
and counterfactual reasoning attribute the outlier scores to ancestor
nodes \citep{budhathoki2022causal}. The causal structure of a system
is a powerful tool enabling the formal analysis of the root cause
of unexpected event \citep{budhathoki2022causal}. It is a directed
acyclic graph (DAG) composed of nodes representing system components
and directed edges representing causal links or dependency connections
from parent nodes to child nodes. The observational data of nodes
are assumed to have additive noise, the edge connections are assumed
to be noise-free \citep{budhathoki2022causal}. However, in computer
networking systems, e.g., the connection between components can be
noisy or faulty due to varying workloads, faulty hardware, wear and
tear, or signal interference. Therefore, the root cause of an outlier
can also arise from an edge in addition to a node. This raises the
question of whether an anomalous edge could be detected by existing
algorithms. In a recent work, \citet{ni2017ranking} attempted to
detect faulty edges in those networks. Nevertheless, to the best of
our knowledge, the study of noisy causal links as root causes has
not been considered.

In this paper, we aim to fill this gap by generalising existing frameworks
to allow the detection of anomalous edges as well as anomalous nodes.
We consider changes at both the individual edges and model them via
a Bayesian linear regression, where the noise in the causal edges
is represented as the distribution of the regression weights.

In addition, we propose a causal contribution score called Bayesian
Integrated Gradient of Edge and Node noise (BIGEN) to attribute a
leaf anomaly score to the ancestor nodes and edges. First, given observations
of outliers, we infer the noise values of the causal edges represented
by Bayesian linear regression coefficients using the mode of the posterior
distribution (MAP estimate) \citep{bishop2006pattern}. Second, we
infer the noise values of the nodes and edges given the outlier observations.
Third, we use an attribution method to compute the contribution of
each noise term to the outlier score. Existing works for explaining
outliers use Shapley-based attribution methods that require summing
over all possible subsets of the ancestor nodes and edges \citep{sundararajan2020many,covert2021improving,budhathoki2022causal}.
This is computationally expensive, even for approximate methods, when
applied to large graphs. In this work, we apply the integrated gradient
(IG) method \citep{sundararajan2020many} along the path from some
references to the target noise for calculating the attribution scores.
We show that this attribution method is efficient for scaling to graphs
with thousands of nodes in attributing the root causes, compared to
the baseline methods based on Shapley values \citep{sundararajan2020many}.

Our contributions are: 

1. A framework for identifying the sources of changes in both nodes
and edges to reduce operational costs. 

2. Modelling the causal connections using Bayesian linear regression
to enable the inference of edge noises and applying the Shapley-based
attribution framework.

3. Introducing a new causal contribution score called BIGEN to efficiently
attribute a leaf anomaly to the ancestor nodes and edges. 

4. Demonstrating the effectiveness of the proposed methods on random
graph datasets and two real-world scenario datasets.

\section{Preliminaries}

\subsection{Outlier Scores}

Based on information theory, \citet{budhathoki2022causal} introduced
an outlier score that calibrates all probabilistic outlier scores.
The authors defined the outlier score by characterising the tail probability
of an event $X=x$ based on the distribution of some score space,
such as negative log likelihood or $z$-score, as follow:
\begin{align}
S_{X}(x) & =-\log P\left\{ -\log p(X)\ge-\log p(x)\right\} \label{eq:log-score}\\
\text{or,~}S_{X}(x) & =-\log P\left\{ |X-\mu_{x}|\ge|x-\mu_{X}|\right\} \label{eq:z-score}
\end{align}

\subsection{Functional Causal Mechanisms (FCMs)}

Given a causal graph represented as a DAG, its functional causal mechanisms
\citep{pearl2009causality,peters2017elements} can be described by
the following set of equations. For each node $j$:
\begin{align*}
X_{j} & =f_{j}(\text{Pa}_{j},\epsilon_{j})\\
\text{or,~\ensuremath{X_{j}}} & =\sum_{i\in PA_{j}}W_{ij}X_{i}+\epsilon_{j},\qquad\text{when \ensuremath{f_{j}} is linear}
\end{align*}
where $PA_{j}$ is the node indices of the parents of node $j$. These
equations represent the dependence of a node $X_{j}$ on its parent
nodes $X_{Pa_{j}}=\{X_{i}:i\in\text{Pa}{}_{j}\}$ and $\epsilon_{j}$
which is an additive noise random variable independent of $X_{i}$.
If we assume $f_{j}$ is linear, given an observation data matrix
$X$, we can fit a linear regression model for this causal model to
learn the weight parameters $W_{ij}$. A leaf outlier $X_{n}=f_{n}(\text{Pa}_{n},\epsilon_{n})$
can be recursively expressed as a function depending on all the noise
variables as follows:
\[
X_{n}=f(\epsilon_{1},\dots,\epsilon_{n})=f(\boldsymbol{\epsilon})
\]
where $P_{\boldsymbol{\epsilon}}=P_{\epsilon_{1}}\times\dots\times P_{\epsilon_{n}}$.

\subsection{Shapley-based RCA of outliers\label{subsec:Shapley-based-RCA-methods}}

\citet{budhathoki2022causal} employed the concept of Shapley values
\citep{shapley1953value} from cooperative game theory and randomised
experiment to measure the contribution of a noise term to the target
outlier score for an observed outlier at the leaf. The Shapley value,
in the context of a coalition game, is an axiom-based method which
uniquely divides (attributes) a pay-off among the players (the noise
variables $\{\epsilon_{j}\}$ in our study) \citep{sundararajan2020many}.

\paragraph{Shapley Classic}

The originally proposed Shapley value \citep{shapley1953value} for
each player $j$ is computed as the weighted average of its marginal
contributions for all possible coalitions (subsets of players), both
with and without the player $j$. Let $P={1,\dots,d}$ denote the
set of players, $Q\subseteq P$ be a subset, $|Q|$ denote its cardinality,
and $v$ be the value function. The Shapley value for the player $j$
is defined as:
\begin{equation}
\phi_{j}(v)=\sum_{Q\subseteq P\backslash j}\frac{|Q|!(|P|-|Q|-1)}{P!}\left(v(Q\cup i)-v(Q)\right)\label{eq:shapley-classic}
\end{equation}

\paragraph{Sampling }

This method approximates the Shapley values as a solution to the weighted
least squares problem \citep{janzing2020feature,covert2021improving}:
\begin{align}
\min_{\phi_{1},\dots,\phi_{n}}\sum_{Q\subseteq P}w(Q)\left(u(Q)-v(Q)\right)\label{eq:shapley-sampling}\\
w(Q)=\frac{d-1}{\left(\begin{array}{c}
d\\
|Q|
\end{array}\right)|Q|(d-|Q|)}\nonumber 
\end{align}
where the weighting function $w(Q)$ depends only on the cardinality
of the subset $Q$, and $u(S)=\sum_{i\in S}\phi_{i}$ represents the
approximation function.

\paragraph{Permutation}

This method approximates the Shapley values using Monte Carlo samples
of data and random permutations \citep{strumbelj2014explaining}.
It first draws a random instance $\epsilon^{m}$ from the data, then
chooses a random permutation of the independent variables, and finally
computes the marginal contribution as:
\begin{equation}
\phi_{j}(v)=\frac{1}{M}\sum_{m=1}^{M}\left(v(\epsilon_{+j}^{m})-v(\epsilon_{-j}^{m})\right)\label{eq:shapley-permute}
\end{equation}
where $v(\epsilon_{+j}^{m})$ is the prediction for $\epsilon$, but
with a random number of variables replaced by the values from $\epsilon^{m}$,
except for $\epsilon_{j}$, and $v(\epsilon_{-j}^{m})$ similarly,
but $\epsilon_{j}$ is also replaced by $\epsilon_{j}^{m}$.

\paragraph{Value function}

The value function $v$ in this study is the outlier score $S(x_{n})$
at a leaf node \citep{budhathoki2022causal}.

\section{Methods}

We first present our framework for RCA of an observed anomaly at a
leaf node.  We formulate a noisy mechanism to model the noise in
the edge connection between component nodes. This allows the leaf
node to be expressed as a function of both node and edge noise dependent
model. Finally, we present the integrated gradient (IG) based contribution
score of ancestor edge and node noises.

\subsection{Framework \label{subsec:Framework}}

For concreteness, let us consider a computer network system\footnote{Our method is applicable to all structural causal processes, e.g.,
in a manufacturing process with sensor readings at each nodes as the
variables.} as a running example. We have continuous observations from the network
components represented as the variables $X_{1},\dots,X_{n}$. The
causal relationships between these component variables are known and
given in the form of a directed acyclic graph (DAG). There is edge
noise in each connection from $X_{j}$ to parent to describe potential
noise between each parent-child pair. As example, this can be a random
delay in the connection between two servers due to a router malfunction
in between. Based on the causal mechanism \citep{peters2017elements},
we introduce noise to each causal link from each parent node $X_{i}\in\text{Pa}{}_{j}$
to $X_{j}$ in the causal graph to create a noisy mechanism, where
$\text{Pa}{}_{j}$ is the set of parent nodes of node $j$.
\begin{defn}
Noisy Causal Mechanisms.\label{def:Noisy-Causal-Mechanism.}

For node $j$ , a noisy mechanism is a generative process,
\begin{align}
\xi_{\cdot j} & \sim\mathcal{N}(0,\alpha^{-1}I)\label{eq:fcm1}\\
W_{\cdot j} & =\mu_{\cdot j}+\xi_{\cdot j}\label{eq:fcm2}\\
\epsilon_{j} & \sim\mathcal{N}(0,\beta^{-1})\label{eq:fcm3}\\
X_{j} & =f_{j}(W_{\cdot j}^{T}X_{\text{Pa}{}_{j}})+\epsilon_{j}\label{eq:noisy-mechanism}
\end{align}
where $f_{j}$ can be a nonlinear function, $\xi_{\cdot j}$ and $\epsilon_{j}$
are Gaussian edge and node noises with fixed variances $\alpha^{-1}$
and $\beta^{-1}I$ ($\alpha$ and $\beta$ are the precision parameters),
$\mu_{\cdot j}$ is the mean of $p(W_{\cdot j})$ with $W_{ij}$ being
the causal link from node $i$ to node $j$ for $i\in\text{Pa}{}_{j}$,
and we use $\cdot$ (dot) to denote the index varying in the set of
parent node $\text{Pa}{}_{j}$.
\end{defn}

We use this noisy mechanism to model the generative process of the
system during normal operations. During testing, suppose some node
or edge of the system behave strangely, either the node noise $\epsilon_{j}$
or some edge noise $\xi_{ij}$ is interfered with, causing anomalous
observations at downstream nodes and leaf nodes.
\begin{defn}
Abnormal Causal Mechanism\label{def:Abnormal-Causal-Mechanism} 

An abnormal causal mechanism of a node $j$ is a similar generative
process as the noisy causal mechanism of the node $j$ in Definition~\ref{def:Noisy-Causal-Mechanism.}
but with the noise distribution $p(\epsilon_{j})$ or $p(\xi_{ij})$
interfered.
\end{defn}

During abnormal operation, as defined in Definition~\ref{def:Abnormal-Causal-Mechanism},
the underlying anomalous noises $\xi_{ij}$ and $\epsilon_{j}$ render
the generative processes defined in Eq.~\ref{eq:noisy-mechanism}
to generate anomalous observations of the system. The RCA module then
collects these anomalies along with their ancestors' observations,
denoted as $X'$, for analysing and identifying the root causes using
an attribution algorithm. Note that this generalises the approach
presented in \citep{budhathoki2022causal} which only considers node
anomalies.

\paragraph{Fitting the Noisy Mechanism}

In this paper, we model a noisy mechanism using Bayesian linear regression,
allowing us to compute its posterior analytically. We use training
data collected during the normal operation of the system to fit the
posterior of the noisy mechanism, as described in the following theorem.

\begin{thm}
Posterior distribution of noisy mechanism \label{thm:Posterior-distribution-of}

For each node $j$, let $(X_{\text{Pa}{}_{j}},X_{j})$ be the set
of input-output observations of the noisy mechanism as defined in
Definition~\ref{def:Noisy-Causal-Mechanism.}. Assuming fixed $\alpha$,
and $\beta$ precision parameters. If the prior for the weights of
the causal links and the regression output are respectively: 
\begin{align*}
W_{\cdot j} & \sim\mathcal{N}\left(\mu_{\cdot j}^{0},\alpha^{-1}I\right)\\
X_{j} & \sim\mathcal{N}\left(W_{\cdot j}^{T}X_{\text{Pa}{}_{j}},\beta^{-1}I\right),
\end{align*}
then the posterior distribution of the weight vector $W_{\cdot j}$
of the mechanism $j$ is given by: 
\begin{align*}
W_{\cdot j}|X_{j} & \sim\mathcal{N}(\mu_{\cdot j},H)\,\,\text{for}\\
\mu_{\cdot j} & =H^{-1}\left(\alpha\mu_{\cdot j}^{0}+\beta X_{\text{Pa}{}_{j}}^{T}X_{j}\right)\\
H & =\alpha I+\beta X_{\text{Pa}{}_{j}}^{T}X_{\text{Pa}{}_{j}}.
\end{align*}
\end{thm}

\begin{proof}
Refer to Section~3.3 in \citep{bishop2006pattern}.
\end{proof}
\begin{prop}
MAP estimate of the noisy mechanism \label{prop:MAP-estimate-of}

Given the data and the prior in Theorem~\ref{thm:Posterior-distribution-of},
the maximum a posteriori (MAP) estimate of the weights $W_{\cdot j}$
is
\[
W_{\cdot j}^{MAP}=\mu_{\cdot j}=H^{-1}(\alpha\mu_{\cdot j}^{0}+\beta X_{\text{Pa}{}_{j}}^{T}X_{j})
\]
\end{prop}

\begin{proof}
By definition, the MAP estimate is the mode of the posterior distribution
\citep{bassett2019maximum}. Since the posterior is a Gaussian, its
mode coincides with its mean.
\end{proof}

\subsection{Noise dependent reparametrisation}

When considering RCA of a leaf outlier $X_{n}=f_{n}(\text{Pa}_{n},\epsilon_{n},\xi_{\cdot n})$,
it is convenient to re-parameterise $X_{n}$ as a function of the
node and edge noises $(\epsilon,\xi)$ to facilitate the attribution
algorithms later on. We can recursively rewrite each ancestor of $X_{n}$
in Eq.~\ref{eq:noisy-mechanism} up to the root node and arrive at
a function of these noises 
\begin{equation}
X_{n}=g(\epsilon,\xi)\label{eq:noise-dependant-funnction}
\end{equation}
where $P_{\boldsymbol{\epsilon}}=\prod_{i}P_{\epsilon_{i}}$, and
$P_{\xi}=\prod P_{\xi_{ij}}$ are independent noises. Using this re-parameterisation,
we can attribute the root cause of these anomalies directly to these
noise variables. Given a batch of abnormal values $X'$ with $x'_{n}$
being the observed anomaly at the leaf node $X_{n}$, we first infer
the noise $\xi'$, and $\epsilon'$ as follows.

\paragraph{Edge noise estimation}

We employ the MAP estimate of the posterior from Proposition~\ref{prop:MAP-estimate-of}
to compute the new edge weights $W_{\cdot j}'$ for each mechanism
$j$. We then use the posterior weights estimated from the training
data as the new prior. Finally, the estimated edge noise is defined
as:
\begin{equation}
\xi_{\cdot j}'=W_{\cdot j}'-W_{\cdot j}\label{eq:edge-noise}
\end{equation}

Note that if $X'$ comes from the same distribution as $X$, i.e.,
the generative process is normal, the new edge weights $W_{\cdot j}'$
will coincide with the mode $W_{\cdot j}$ of the fitted edge weight
posterior, rendering the edge noise being close to a zero vector.

\paragraph{Node noise estimation}

We use Eq.~\ref{eq:noisy-mechanism} to estimate the node noise as
\begin{align}
\epsilon_{j} & =X_{j}-f_{j}\left(W_{\cdot j}^{T}X_{\text{Pa}{}_{j}}\right).\label{eq:node-noise}
\end{align}

\subsection{Integrated Gradient of Edge and Node noise \label{subsec:Intergrated-Gradient-of}}

We introduce a Bayesian Integrated Gradient of Edge and Node noise
(BIGEN in short) to attribute a leaf anomaly score to ancestor nodes
and edges. Instead of using subsets as in Shapley values methods,
we use noise reference to explain the root cause. This approach offers
a performance gain compared to previous methods. Specifically, we
select a reference node noise $\epsilon'$ from the normal dataset
and the mean edge noise $\xi'=W$. Additional reference points can
also be chosen to calculate multiple contribution scores, followed
by averaging, to enhance the accuracy of the score.

Let $f$ be the score function in Eq.~\ref{eq:log-score} which is
a continuous and differentiable function. Then $\phi_{i}(x,x',f)$
is defined as the integral of the gradient of $f$ along the straight-line
path between $x$ and $x'$. Formally, we define the path between
$x$ and $x'$ as $\gamma(t)=tx+(1-t)x'$ for $t\in[0,1].$ Then,
the Integrated Gradient (IG) for the $i$th feature $x_{i}$ is defined
as: 
\begin{align}
\text{IG}{}_{i}(x,x',f) & =(x-x')\int_{t=0}^{1}\frac{\partial f(\gamma(t))}{\partial\gamma_{i}(t)}\frac{\partial\gamma_{i}(t)}{\partial t}dt\nonumber \\
 & =(x_{i}-x_{i}')\int_{t=0}^{1}\frac{\partial f(x'+t(x-x'))}{\partial x_{i}}dt\label{eq:integ-grad}
\end{align}

This score integrates the gradient along the path from the reference
to the inferred noise to calculate the attribution of each node and
edge towards the observed anomaly score. In our case, with two noise
variables $\epsilon$ and $\xi$, computing the IG for one requires
marginalising over the other noise variable. Therefore, we adjust
this IG to fit into our attribution to node and edge noises as follows:
\begin{align}
\text{IG}{}_{i}(\epsilon,\epsilon',f) & =\mathbb{E}_{\xi}\left[\epsilon_{i}-\epsilon_{i}'\right]\int_{t=0}^{1}\frac{\partial f_{t}}{\partial\epsilon_{i}}dt\label{eq:integ-node}\\
\text{IG}{}_{ij}(\xi,\xi',f) & =\mathbb{E}_{\epsilon}\left[\xi_{ij}-\xi_{ij}'\right]\int_{t=0}^{1}\frac{\partial f_{t}}{\partial\xi_{ij}}dt\label{eq:integ-egde}
\end{align}

This method is more advantageous than subset sampling, as it is linear
and only dependent on the number of discretised steps in the path.
In contrast, Shapley-based method \citep{sundararajan2020many} requires
summing over all possible subsets of the ancestor nodes and edges,
which grows exponentially with the number of nodes and edges.

\paragraph{Score function}

Since the joint distribution of nodes and edges is Gaussian, the conditional
distribution of the normal observations at the leaf node is also Gaussian.
In this case, the outlier score using the negative log-likelihood
feature in Eq.~\ref{eq:log-score} is equivalent to the outlier score
with the $z$-score feature in Eq.~\ref{eq:z-score}. An efficient
evaluation function can be derived using the error function as follows:
\begin{align}
S_{X}(x) & =-\log P\left\{ -\log p(X)\ge-\log p(x)\right\} \nonumber \\
 & =-\log P\left\{ \frac{|X-\mu_{X}|^{2}}{2\sigma_{X}^{2}}\ge\frac{|x-\mu_{X}|^{2}}{2\sigma_{X}^{2}}\right\} \nonumber \\
 & =-\log P\left\{ |\frac{X-\mu_{X}}{\sigma_{X}}|\ge|\frac{x-\mu_{X}}{\sigma_{X}}|\right\} \nonumber \\
 & =-\log\{1-\text{\ensuremath{\Phi}}(z)\}=-\log\text{\ensuremath{\Phi}}(-z)\label{eq:normal-cdf}
\end{align}
where $z=|\frac{x-\mu_{X}}{\sigma_{X}}|$, $\text{\text{\ensuremath{\Phi}}}(z)=\frac{1}{\sqrt{2\pi}}\int_{-\infty}^{z}e^{-\frac{t^{2}}{2}}dt$
is the standard normal cumulative distribution function, and $(\mu_{X},\sigma_{X})$
represents the maximum likelihood estimate of the marginal mean and
standard deviation of the marginal distribution of the target node.

\paragraph{Causal graph and noisy FCM}

While our focus is not on solving causal discovery, we assume a causal
graph is given. For each fixed weight noise vector, we have a set
of FCMs with additive node noise and their FCM parameters can be learned
from normal observation data \citep{peters2017elements}. \citet{budhathoki2022causal}
observed that when representing the FCM of a target leaf node as a
functional of all node noise, each noise vector takes on the role
of selecting a deterministic mechanism. We generalise this idea and
represent the FCMs as functionals of the edge and node noises, which
similarly play the role of choosing the deterministic mechanisms.
Inferred noises of outlier edges or nodes, therefore, will select
outlier mechanisms in either cases. We build upon the success of \citet{budhathoki2022causal}'s
approach, to infer functions and noise from data, and show that counterfactual
contribution score by changing the noise term w.r.t. a reference is
effective for causal attribution at both the node and edge levels.
This approach utilises Pearl's third ladder of causation \citep{pearl2009causality}
to enhance our understanding of the system's causes and effects.

\paragraph{Shapley values and IGs}

\citet{budhathoki2022causal} computed Shapley value contributions
numerically by averaging over all orderings sets. However for larger
number of variables, approximation is needed to be practical. In
the experiments, we sample orderings instead of using all orderings
and compare Shapley with early stopping (Shapley), subset sampling
(Sampling), and methods based on a fixed number of randomly generated
permutations (Permutation). For the IGs, however, no subsets of intervention
are needed but rather a gradient path is taken along the path from
some reference noises to the target noises. For each noise vector,
BIGEN requires one forward pass and one backward pass. Since we use
a small fixed number of references, this computation scales linearly
with the number of nodes and edges in the subgraph. One advantage
of BIGEN over Shapley methods is that it can be applied to nonlinear
(noisy) FCMs since the contribution score is still linear w.r.t. the
gradients, thus satisfying the efficiency axiom \citep{shapley1953value}.

\section{Related Work}

 Causal structure discovery-based techniques have been recently
used to find the root cause(s) of faults in cloud applications \citep{arnold2007temporal,wang2018cloudranger}.
More recently, \citet{budhathoki2022causal} introduced a causal structure
based root cause explanation for outlier using Shapley values. The
authors represented the dependence of an outlier leaf directly in
terms of the noise of ancestor nodes. This approach facilitates intervention
on the noise distributions. Their work, however, assumes that the
outlier does not originate from the connections in the functional
causal models. In real computer network systems, e.g., the connections
may be broken or malfunction, causing anomalies in downstream nodes.
In our present work, we relax this constraint and consider also changes
in the weights as one possible source of the root cause, in addition
to the node anomalies.

\citet{budhathoki2021did} model a distribution change of a mechanism
for certain nodes and use a new dataset under this new mechanism to
fit a new model. In contrast, our work models the uncertainty in the
edges of the mechanisms and allows for temporary anomaly edge noise
causing a faulty batch of observations. Under Bayesian view, we employ
the MAP estimate of the faulty edge noise deviation from the normal
edge prior. By estimating noise in both edges and nodes, we can quantify
the contribution of each noise term to the outlier score observed
at the leaf. This can be viewed as intervention and counterfactual
analysis on the nodes and edges. To the best of our knowledge, we
are the first to explain anomalies in both nodes and edges based on
the given causal structure. A different approach from our work, which
combines RCA and causal discovery in one framework, is proposed by
\citet{ikram2022root}. 

In the recent domain of explainable AI, methods based on Shapley value
\citep{shapley1953value,sundararajan2020many} have gained increasing
popularity. These methods use an axiomatic approach to design attribute
functions with desirable properties, e.g., being fair, unique, and
efficient. These methods explain prediction outcomes by attributing
the prediction score back through the deep neural networks to the
input features \citep{erion2021improving,yang2022re}. Among them,
integral-based attribution methods are the most efficient, as they
use only a reference to represent the absence of the input signal,
rather than a random feature from the training data \citep{lundberg2017unified,sundararajan2020many,samek2021explaining}.
However, these methods are designed for explaining neural network
predictions, which differs from our goal of RCA.

\section{Experiments}

We run experiments on random graph datasets and two real-world settings,
namely a micro cloud service and a supply chain scenarios. We compare
our methods against three baselines described in the Preliminaries
section, as well as a naive approach:
\begin{enumerate}
\item \textbf{Shapley} (classic): This method employs Shapley values as
defined in Eq.~\ref{eq:shapley-classic} for the contribution of
each node.
\item \textbf{Sampling}. This method calculates the Shapley values by selecting
random subsets and weighted least squares regression \citep{janzing2020feature,covert2021improving},
as indicated in Eq.~\ref{eq:shapley-sampling}.
\item \textbf{Permutation}: This method computes the Shapley values through
permutation sampling \citep{strumbelj2014explaining}, as shown in
Eq.~\ref{eq:shapley-permute}.
\item \textbf{Naive}: This method uses the marginal distribution, i.e.,
the observational distribution, of each node $X_{j}$ to compute the
Shapley values.
\item \textbf{BIGEN} (ours): We use the contribution scores outlined in
Eq.~\ref{eq:integ-node}, \ref{eq:integ-egde} to assess the node
and edge contributions.
\end{enumerate}
For all methods, we use the outlier score in Eq.~\ref{eq:normal-cdf}
and assess the contribution to this score by each ancestor node and
edge. For the edge score of the baselines, we use the outer product
of the node scores to estimate the edge score $s_{edge}=s_{node}s_{node}^{T}$,
where $s_{node}=(s_{1},\dots,s_{n})$ represents the contributions
of all the ancestor nodes (including the target node). This score
quantifies the level of anomaly associated with an edge $e_{ij}$
by combining the scores observed at nodes $i$ and $j$.

We adhere to the approach of \citet{budhathoki2022causal} and employ
NDCG@$k$ \citep{jarvelin2017ir} to gauge rankings based on graded
relevance of outcomes. NDCG@$k$ yields values within the $[0,1]$
range, with higher scores indicating that highly ``relevant'' root
causes are assigned higher ranks. To establish the ground truth relevance
of all nodes, we assign zero relevance scores to non-root causes,
and invert the ranking of injected root causes, akin to the methodology
in \citep{budhathoki2022causal}.

\subsection{Random graph datasets}

We randomly generate 1,000 causal graphs with varying number of nodes
in the range from 10 to 10000 nodes. The noisy causal mechanisms follow
Definition~\ref{def:Noisy-Causal-Mechanism.} with $\alpha_{j}^{-1}=1$,
$\beta_{ij}^{-1}=0.01$, and $\mu_{ij}\sim|\mathcal{N}(0,1)|$. We
draw normal data from these noisy FCMs, following their generative
process. For the abnormal data, we randomly select a target node $X_{n}$
from this causal graph. We then choose among its ancestor node either
$k\in[1,\dots,m]$ root-cause nodes, or $l\in[1,\dots,m]$ root-cause
edges, or $k+l$ root-cause nodes and edges. Here, $m$ is chosen
to be $10\%$ of the number of nodes in the subgraph. We inject outlier
noises into the nodes and edges to create the ground truths as follows:
The outlier node noise $\epsilon_{j}$ of node $j$ is randomly drawn
from $\mathcal{N}(a,b)$, where $a$ is drawn from $\pm\text{Uniform}(3,5)$
and $b$ is drawn from $\text{Uniform}(3,5)$. The outlier edge noise
$\xi_{ij}$ of each node $j$ is randomly drawn from $\mathcal{N}(am_{j},bs_{j})$,
where $a$ is drawn from $\pm\text{Uniform}(3,5)$, $b$ is drawn
from $\text{Uniform}(3,5)$, and $m_{j}$ represents the maximum magnitude
of the current weights $w_{ij}$ (i.e., $m_{j}=\max_{i}(|w_{ij}|)$),
and $s_{j}$ represents the maximum standard deviation $\sigma_{ij}$
(i.e., $s_{j}=\max_{i}(\sigma_{ij})$).

Fig.~\ref{fig:comparing-5-methods-3-tasks} presents a comparison
of the results obtained from all methods. On average, BIGEN outperforms
all the baselines in detecting the actual root causes of outliers
in the top $k$ candidates, at over NDCG 0.9. The Shapley, Sampling,
and Permutation methods can overall detect root causes of outliers,
at around NDCG 0.83, with similar performance across them. The Shapley
method seems to perform better than Sampling and Permutation in node
and edge anomaly attributions on average. The Naive method, on the
other hand, could not attribute the root causes correctly. Table~\ref{tab:random-graphs}
shows the details of node and edge attribution results for the case
when both type of anomaly are present. Overall, the edge score rankings
are on average lower than that of the node scores. This suggests that
it is greater difficulty in attributing root causes when edge anomalies
are present.
\begin{figure}
\begin{centering}
\includegraphics[width=0.9\columnwidth]{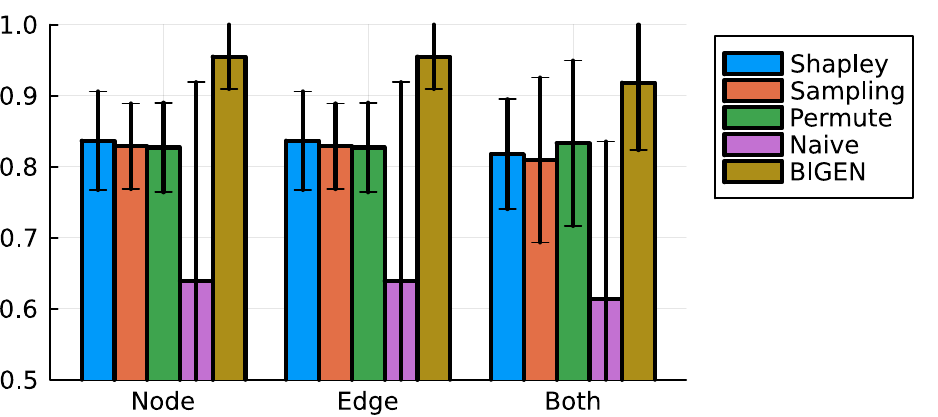}
\par\end{centering}
\caption{NDCG@$k$ ranking of root cause detection on random graphs.\label{fig:comparing-5-methods-3-tasks}}
\end{figure}
\begin{table}[H]
\begin{centering}
\begin{tabular}{l|c|c|c}
\hline 
\multirow{1}{*}{} & {\footnotesize{}Nodes} & {\footnotesize{}Edges} & {\footnotesize{}Nodes + Edges}\tabularnewline
\hline 
{\footnotesize{}Shapley{*}} & {\scriptsize{}$0.815\pm0.045$} & {\scriptsize{}$0.820\pm0.110$} & {\scriptsize{}$0.818\pm0.077$}\tabularnewline
{\footnotesize{}Sampling} & {\scriptsize{}$0.847\pm0.095$} & {\scriptsize{}$0.772\pm0.137$} & {\scriptsize{}$0.809\pm0.116$}\tabularnewline
{\footnotesize{}Permut.} & {\scriptsize{}$0.845\pm0.115$} & {\scriptsize{}$0.820\pm0.118$} & {\scriptsize{}$0.833\pm0.116$}\tabularnewline
{\footnotesize{}Naive} & {\scriptsize{}$0.691\pm0.134$} & {\scriptsize{}$0.536\pm0.303$} & {\scriptsize{}$0.633\pm0.222$}\tabularnewline
{\footnotesize{}BIGEN} & {\scriptsize{}$\mathbf{0.898\pm0.085}$} & {\scriptsize{}$\mathbf{0.938\pm0.105}$} & {\scriptsize{}$\mathbf{0.918\pm0.095}$}\tabularnewline
\hline 
\end{tabular}
\par\end{centering}
\caption{NDCG@$k$ for root cause detection in random graphs. ({*}) Due to
exponential runtime, for \textgreater 20 nodes, we compute Shapley
values using early stopping when the contribution score does not change
much \citep{dowhy_gcm}. \label{tab:random-graphs}}
\end{table}

\paragraph{Runtime}

Next, we compare the runtime of BIGEN to Shapley-based methods. Fig.~\ref{fig:runtime}
shows the runtime of all methods on a Ubuntu 20.04 workstation with
an Intel Xeon E5-1650 v4 CPU and 46Gb RAM. Notably, both the BIGEN
and Naive RCA methods exhibit linear time complexities relative to
the number of upstream nodes, completing computations in less than
a minute for target nodes with roughly 200 ancestors. In contrast,
the Shapley method demonstrates exponential complexity, requiring
over 2 hours for cases involving more than 20 nodes. This disparity
arises due to the Shapley method's necessity to iterate through all
conceivable subsets for the computation of the weighted average of
each subset's contribution score. The Sampling and Permutation methods
tend to exhibit polynomial time complexities, taking just slightly
over 1 hour to compute for a target outlier node featuring around
200 ancestor nodes. The Naive method independently computes the contribution
of each node, while BIGEN capitalises on gradient information and
baseline noises to compute the contribution of each noise term within
the given context.
\begin{figure}
\begin{centering}
\includegraphics[width=0.9\columnwidth]{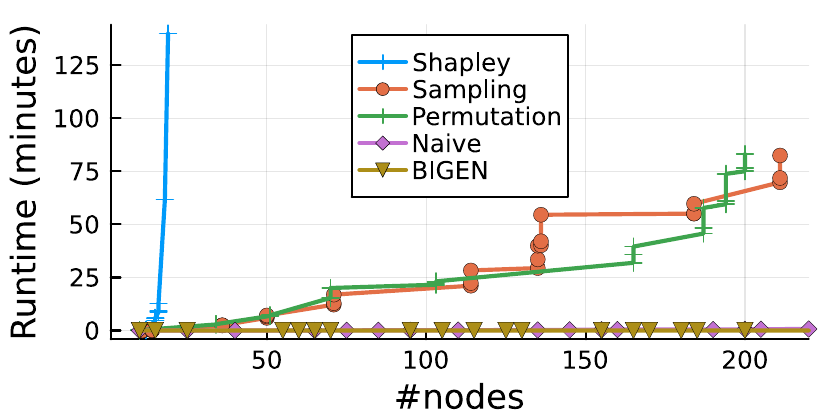}
\par\end{centering}
\caption{Comparing wall-clock runtime (in minutes) between BIGEN and Causal-RCA
methods with increasing number of nodes. The attribution complexity
of BIGEN and Naive is $O(d+e)$, while Shapley is $O(2^{(d+e)})$.
Sampling and Permutation methods tend to have polynomial time. \label{fig:runtime}}

\end{figure}

The key takeaway from this experiment is that BIGEN provides relevant
top-$k$ ranking across values of $k$ while reducing the computation
time significantly. This is because BIGEN can model the uncertainty
in the causal edges to handle situations with noisy mechanisms.

\subsection{Root causes of observed latencies in cloud services}

In this experiment, we study the root causes of unexpected observed
latencies in a microservice of an online shop. Microservices are building
blocks for complex mobile and Internet-of-Things (IoT) applications.
Therefore, it is essential to keep service access delay minimal \citep{guo2022joint}.
Any unusual delay should be quickly studied to identify the root causes
and resolve the problem \citep{ikram2022root}. Often, the delay at
a node is composed of the communication delay between itself and each
parent node (due to network bandwidth, load, and other turbulences)
and execution delay of each node (due to the request complexity, node
computational power and load). 

We use the microservice architecture described in \citep{dowhy_gcm}.
We assume the delay noise in each node $X_{j}$ is a Gaussian noise
with unknown variance $\alpha_{j}^{-1}$ and in each edge $W_{ij}$
is a Gaussian noise with unknown variance $\beta_{ij}^{-1}$. We thus
make a realistic assumption and allow for noisy connections to account
for the above scenarios. The task is to explain the root cause of
an unwanted observed latency at the customer end in processing an
online order. This service involves multiple other web-services described
by a causal dependency graph involving ten other services \citep{dowhy_gcm},
part of the causal graph is shown in Fig.~4 in the Appendix. Assuming
that we observe latencies in the order confirmation of the Website
leaf node, we also assume all services are synchronised.

For the abnormal data of the target node $\texttt{Website}$ we select
among its ancestor node either $k\in[1,\dots,3]$ root-cause nodes
or $l\in[1,\dots,3]$ root-cause edges or $k+l$ root-cause nodes
and edges, then we inject anomalous noises by random sampling the
noise outside $3\sigma$ ranges of their normal operation distributions
to create the ground truths.

We use similar models as in the previous section for this experiment,
i.e., assuming the causal graph is given then fit the noisy FCMs models
to the training data collected during normal operations. 

\paragraph{Results}

Table~\ref{tab:ranking-micro} displays the outcomes achieved by
all methods, revealing the NDCG@k values for root cause detection
within a micro cloud service environment. It shows that the Shapley-based
methods can detect the root cause in both the nodes and edges. However,
the edge score ranking is lower than that of the node score. This
suggests that the deterministic weights methods fall short in high
weight noise applications. However, the Shapley method shows better
attribution results than Sampling and Permutation. The Naive method
can detect node root causes but not edges root causes. The BIGEN,
in contrast, can model the uncertainty in the causal links therefore
can account better the contribution of the root causes to the observed
outliers at leaves. Overall, BIGEN shows better root cause attributions
across nodes, edges, and both compared to the baselines.
\begin{table}[H]
\begin{centering}
\begin{tabular}{l|c|c|c}
\hline 
\multirow{1}{*}{} & {\footnotesize{}Nodes } & {\footnotesize{}Edges } & {\footnotesize{}Nodes + Edges}\tabularnewline
\hline 
{\footnotesize{}Shapley} & {\scriptsize{}$0.904\pm0.047$} & {\scriptsize{}$0.778\pm0.091$} & {\scriptsize{}$0.841\pm0.069$}\tabularnewline
{\footnotesize{}Sampling} & {\scriptsize{}$0.918\pm0.050$} & {\scriptsize{}$0.715\pm0.140$} & {\scriptsize{}$0.816\pm0.095$}\tabularnewline
{\footnotesize{}Permut.} & {\scriptsize{}$0.914\pm0.048$} & {\scriptsize{}$0.726\pm0.112$} & {\scriptsize{}$0.820\pm0.080$}\tabularnewline
{\footnotesize{}Naive } & {\scriptsize{}$0.797\pm0.102$} & {\scriptsize{}$0.465\pm0.199$} & {\scriptsize{}$0.631\pm0.150$}\tabularnewline
{\footnotesize{}BIGEN } & {\scriptsize{}$0.910\pm0.060$} & {\scriptsize{}$\mathbf{0.924\pm0.076}$} & {\scriptsize{}$\mathbf{0.917\pm0.068}$}\tabularnewline
\hline 
\end{tabular}
\par\end{centering}
\caption{NDCG@$k$ for RCA in a micro cloud service.\label{tab:ranking-micro}}
\end{table}

Fig.~\ref{fig:varying-k-micro-service} shows the NDCG raking at
different $k$ between the methods. It shows that BIGEN can rank the
root cause better on average than all methods for across top-$k\ge2$
values. The increasing NDCG@$k$ scores of all methods with larger
$k$ values show that the relevant root causes have more chance to
appear in the top-$k$ results.
\begin{figure}
\begin{centering}
\includegraphics[width=0.9\columnwidth]{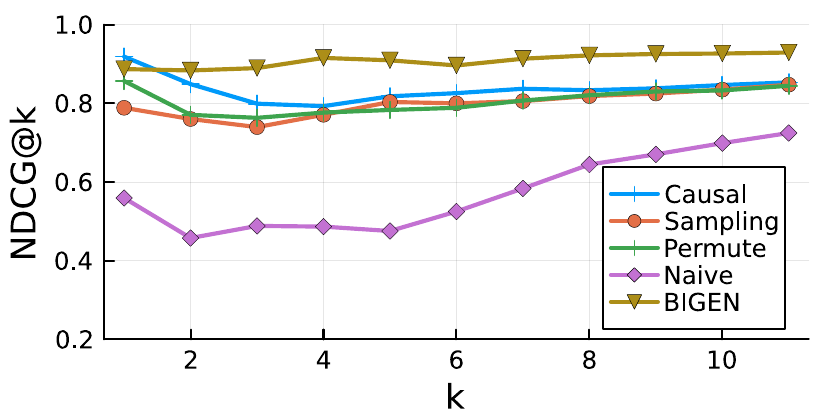}
\par\end{centering}
\caption{NDCG@$k$ when varying the number of $k$, error bars not included
to keep the figure uncluttered.\label{fig:varying-k-micro-service}}
\end{figure}

\subsection{Root Causes of Outliers in a Supply Chain}

In this experiment, we apply our methods to model the noisy interactions
between businesses. The behaviour in supply chains usually includes
complex interaction of organisation structure, and time delays between
decision and implementation \citep{power2005supply}. To gain competitive
advantage, it is critical that businesses need to optimise their supply
chain operations to ensure the flow of physical goods between trading
partners. Thus, it is important that delays to this flow can be identified
and rectified in a reliable and timely manner. Often, business decisions
are made based on predictions of real demands and constraints and
are not a deterministic process \citep{power2005supply}. 

We study RCA of outliers in a supply chain process shown in Fig.~5
in the Appendix. In this process, a retailer need to submit orders
to a vendor based on its inventory constraint of current stock and
the demand forecasting of near future sale to fulfil future customers'
purchases faster. The vendors will then confirm the retailer's purchase
orders with a random delay. When these orders are confirmed, the goods
will be sent to the retailer again with some random delay and may
arrive at different times. The random delays can be due to various
overhead costs such as decisions of the vendors' manager, variable
packaging and shipping operations \citep{croson2014order}. We assume
noisy linear FCMs for the supply chain process where the noise term
of each node is a Gamma distribution to mimic real-world settings,
i.e., heavily-tailed behaviour. We assume Gaussian noises for the
edges to account for the fluctuations in the causal links between
nodes.

For the abnormal data of the target node $\texttt{Received}$ we randomly
select among its ancestor nodes either $k\in\{1,2\}$ nodes or $l\in\{1,2\}$
edges or $k=1,l=1$ node and edge, then inject anomalous noises by
random sampling the node noise $\epsilon\sim\text{Uniform}(3,5)$
and edge connection noise $\xi\sim\text{Uniform}(3,5)$ then collect
the outliers.

\paragraph{Results}

Table~\ref{tab:ranking-supply} presents the outcomes of all methods
employed for root cause detection within a supply chain context. The
result shows good performance among all methods with our proposed
method BIGEN has highest detection ranking followed by Shapley methods,
then Naive-RCA. Notably, our method can accurately identify the anomalous
link in the process, at NDCG@$k$ of $0.98$ on average. This confirms
the effectiveness of the noisy FCM models and the BIGEN attribution
method.
\begin{table}[H]
\begin{centering}
\begin{tabular}{l|c|c|c}
\hline 
\multirow{1}{*}{} & {\footnotesize{}Nodes} & {\footnotesize{}Edges} & {\footnotesize{}Nodes + Edges}\tabularnewline
\hline 
{\footnotesize{}Shapley} & {\scriptsize{}$0.891\pm0.049$} & {\scriptsize{}$0.927\pm0.072$} & {\scriptsize{}$0.909\pm0.061$}\tabularnewline
{\footnotesize{}Sampling} & {\scriptsize{}$0.890\pm0.050$} & {\scriptsize{}$0.925\pm0.074$} & {\scriptsize{}$0.908\pm0.062$}\tabularnewline
{\footnotesize{}Permut.} & {\scriptsize{}$0.892\pm0.049$} & {\scriptsize{}$0.926\pm0.072$} & {\scriptsize{}$0.909\pm0.061$}\tabularnewline
{\footnotesize{}Naive} & {\scriptsize{}$0.856\pm0.055$} & {\scriptsize{}$0.915\pm0.091$} & {\scriptsize{}$0.886\pm0.073$}\tabularnewline
{\footnotesize{}BIGEN} & {\scriptsize{}$0.890\pm0.052$} & {\scriptsize{}$\mathbf{0.980\pm0.045}$} & \textbf{\scriptsize{}$\mathbf{0.935\pm0.048}$}\tabularnewline
\hline 
\end{tabular}
\par\end{centering}
\caption{NDCG@$k$ for RCA in a supply chain.\label{tab:ranking-supply}}
\end{table}

\section{Conclusions}

We have introduced a framework for identifying the root causes of
unexpected events observed at leaf nodes in causal generative processes.
Through modelling noises in both nodes and edges, we proposed noisy
functional causal models that enable the inference of both types of
noises, making them suitable for the application of attribution algorithms.
Furthermore, we have developed an efficient attribution score based
on integrated gradient, which can be readily applied to graphs with
thousands of nodes. Our experimental results demonstrated the effectiveness
of our proposed methods.

\noindent \bibliography{causal}

\end{document}